%% file: main.tex
\pdfoutput=1

\documentclass[11pt]{article}

\usepackage{style/emnlp2021}

\usepackage{times}
\usepackage{latexsym}

\usepackage[T1]{fontenc}

\usepackage[utf8]{inputenc}

\usepackage{microtype}

\newif\ifcomment\commenttrue

\input{style/preamble}

\newcommand\sqa{\textsc{SQA}\xspace}
\newcommand\tabfact{\textsc{TabFact}\xspace}
\newcommand\wtq{\textsc{WikiTQ}\xspace}

\newcommand\hqa{\textsc{HybridQA}\xspace}
\newcommand{\hybrider}{\textsc{hybrider}\xspace}
\newcommand{\pointr}{\textsc{PointR}\xspace}
\newcommand{\etc}{\textsc{etc}\xspace}
\newcommand{\tableetc}{\textsc{TableEtc}\xspace}
\newcommand{\bigbird}{\textsc{BigBird}\xspace}
\newcommand{\tapas}{\textsc{Tapas}\xspace}
\newcommand{\linformer}{\textsc{Linformer}\xspace}
\newcommand{\bert}{\textsc{Bert}\xspace}
\newcommand{\sat}{\textsc{SAT}\xspace}
\newcommand{\tabert}{\textsc{TaBert}\xspace}
\newcommand{\reformer}{\textsc{Reformer}\xspace}
\newcommand{\model}{\textsc{Mate}\xspace}

\title{MATE: Multi-view Attention for Table Transformer Efficiency}

\author{
Julian Martin Eisenschlos$^1$,
Maharshi Gor$^2$\thanks{\,\,\,Work done at Google Research.},
Thomas M{\"u}ller$^3$\footnotemark[1],
William W. Cohen$^1$\\
Google Research$^1$ \\
\texttt{\{eisenjulian,wcohen\}@google.com}\\
Dept. of Computer Science, University of Maryland$^2$ \\
\texttt{mgor@cs.umd.edu}\\
Symanto Research, Valencia, Spain$^3$ \\
\texttt{thomas.mueller@symanto.com}
}

\begin{document}

\newtheorem{theorem}{Theorem}
\newtheorem*{theorem*}{Theorem}
\newtheorem{lemma}[theorem]{Lemma}

\maketitle

\begin{abstract}
    \input{sections/00-abstract}
\end{abstract}

\input{sections/10-intro}
\input{sections/20-related}
\input{sections/30-model}

\input{sections/40-experiments}
\input{sections/50-results}
\input{sections/60-analysis}
\input{sections/70-conclusion}

\bibliographystyle{style/acl_natbib}
\bibliography{bib/journal-full,bib/jbg}

\clearpage

\input{sections/99-appendix}

\end{document}

%% file: style/preamble.tex
\usepackage{amsthm}

\usepackage{framed}
\usepackage{mdwlist}
\usepackage{latexsym}
\usepackage{colortbl}
\usepackage{xcolor}
\usepackage{nicefrac}
\usepackage{booktabs}
\usepackage{amsfonts}
\usepackage[T1]{fontenc}
\usepackage{bold-extra}
\usepackage{amsmath}
\usepackage{amssymb}
\usepackage{bm}
\usepackage{graphicx}
\usepackage{mathtools}
\usepackage{microtype}
\usepackage{multirow}
\usepackage{multicol}
\usepackage{xspace}
\usepackage{latexsym,comment}
\usepackage{pifont}%
\usepackage{listings}

\newcommand{\gem}[1]{\mbox{\textsc{gem}}}
\newcommand{\abr}[1]{\textsc{#1}}

\newcommand{\hidetext}[1]{}
\newcommand{\ignore}[1]{}

\ifcomment
\newcommand{\pinaforecomment}[3]{\colorbox{#1}{\parbox{.8\linewidth}{#2: #3}}}
\else
\newcommand{\pinaforecomment}[3]{}
\fi

\newcommand{\smallurl}[1]{ \begin{tiny}\url{#1}\end{tiny}}

\definecolor{lightblue}{HTML}{3cc7ea}
\definecolor{CUgold}{HTML}{CFB87C}
\definecolor{grey}{rgb}{0.95,0.95,0.95}
\definecolor{ceil}{rgb}{0.57, 0.63, 0.81}
\definecolor{UMDred}{HTML}{ed1c24}
\definecolor{UMDyellow}{HTML}{ffc20e}

\newcommand{\qa}[0]{\abr{qa}}

\newcommand{\cmark}{\ding{51}}
\newcommand{\xmark}{\ding{55}}

\newcommand{\entity}[1]{\underline{#1}}

\newif\ifsubscripterror\subscripterrortrue
\ifsubscripterror
\newcommand{\err}[1]{\textsubscript{~$\pm$#1}}
\else
\newcommand{\err}[1]{ $\pm$ #1}
\fi

%% file: sections/00-abstract.tex
This work presents a sparse-attention Transformer architecture for modeling documents that contain large tables.
Tables are ubiquitous on the web, and are rich in information. However,
more than 20\% of relational tables on the web have 20 or more rows~\cite{cafarella-2008-relweb}, and these large tables present a challenge for current Transformer models, which are typically limited to 512 tokens.
Here we propose \model, a novel Transformer architecture designed to model the structure of web tables. \model uses sparse attention in a way that allows heads to efficiently attend to either rows or columns in a table.
This architecture scales linearly with respect to speed and memory, and can handle documents containing more than 8000 tokens with current accelerators.
\model also has a more appropriate inductive bias for tabular data, and sets a new state-of-the-art for three table reasoning datasets.
For \hqa~\cite{chen-etal-2020-hybridqa}, a dataset that involves large documents containing tables, we improve the best prior result by $19$ points.

%% file: sections/10-intro.tex
\section{Introduction}
\label{sec:intro}

The Transformer architecture~\cite{vaswani17} is expensive to train and run at scale, especially for long sequences, due to the quadratic asymptotic complexity of self-attention.  Although some work addresses this limitation \cite{ainslie-etal-2020-etc, kitaev2020reformer, zaheer2020bigbird}, there has been little prior work
on scalable Transformer architectures for \emph{semi-structured} text.\footnote{The term "semi-structured text" refers to text that has structure that does not reflect a known data schema.  Typically semi-structured text is organized as an \abr{html} tree or variable length lists and tables.}  However, although some of the more widely used benchmark tasks involving semi-structured data have been restricted to moderate size tables, many semi-structured documents are large: more than 20\% of relational tables on the web have 20 or more rows~\cite{cafarella-2008-relweb}, and would pose a problem for typical Transformer models.

\begin{figure}
    \centering
    \includegraphics[width=.84\linewidth]{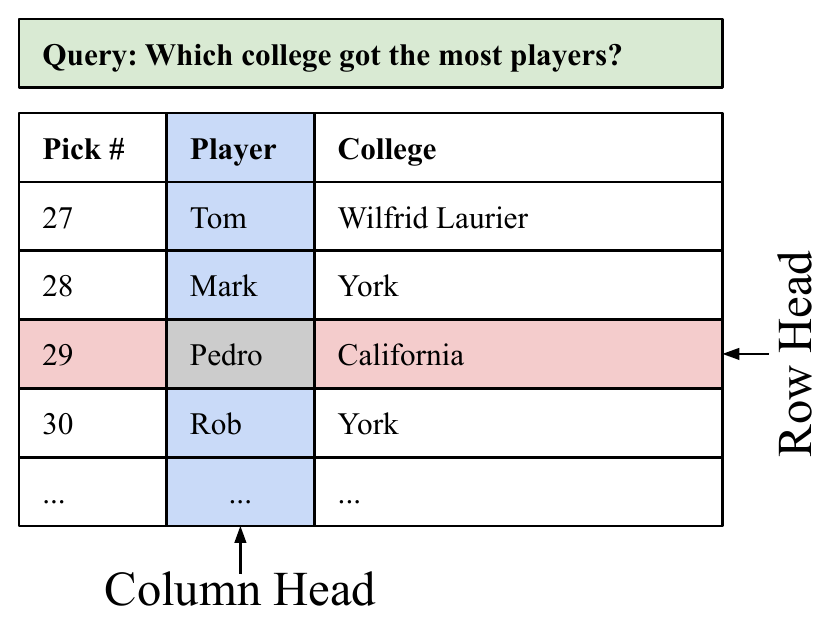}
    \caption{Sparse self-attention heads on tables in \model are of two classes: \emph{Row heads} attend to tokens inside cells in the same row, as well as the query. \emph{Column heads} attend to tokens in the same column and in the query. Query tokens attend to all other tokens.}
    \label{fig:attention}
\end{figure}

Here we study how efficient implementations for transformers can be tailored to semi-structured data.
Figure~\ref{fig:attention} highlights our main motivation through an example:
to obtain a contextual representation of a cell in a table, it is unlikely that the information in a completely different row and column is needed.

We propose the \model architecture\footnote{Pronounced \emph{mah-teh}, as in mate tea.} (Section~\ref{sec:models}), which allows each attention head to reorder the input so as to traverse the data by multiple points of view, namely column or row-wise (Figure~\ref{fig:efficiency}).
This allows each head to have its own data-dependent notion of locality, which enables the use of sparse attention in an efficient and context-aware way.

\begin{figure*}[ht]
    \centering
    \includegraphics[width=.95\linewidth]{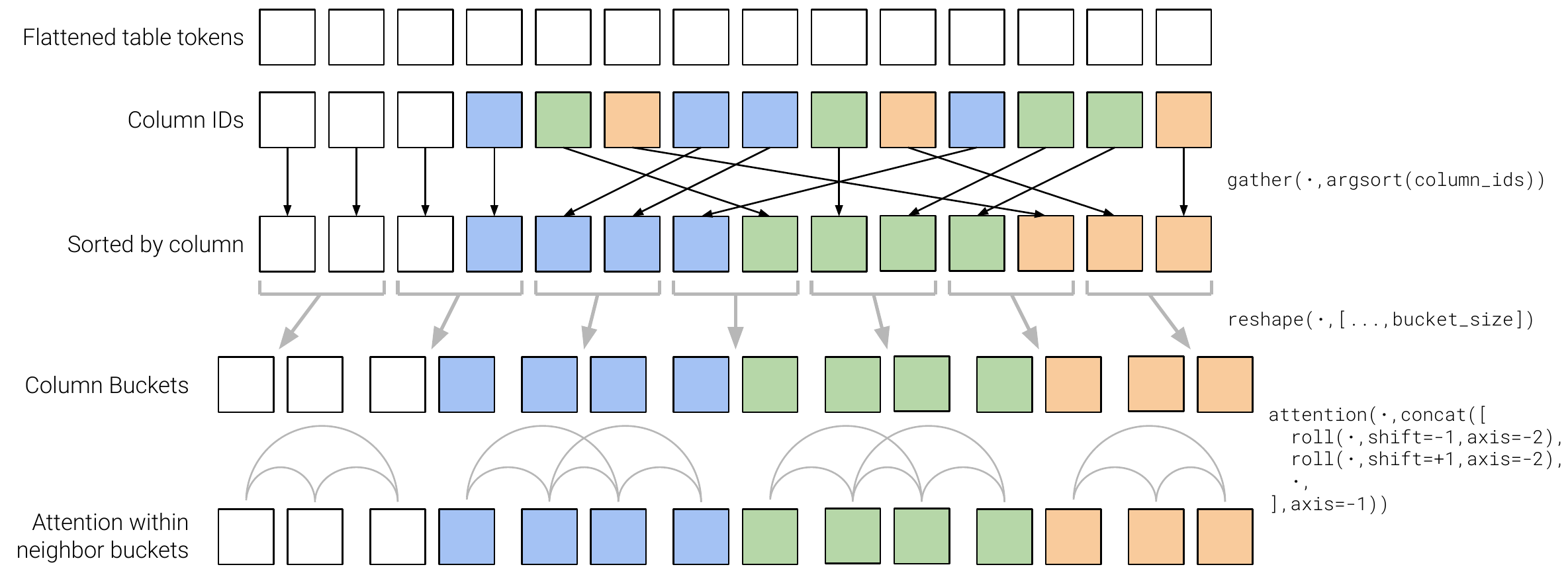}
    \caption{Efficient implementation for \model. Each attention head reorders the tokens by either column or row index and then applies a windowed attention mechanism. This figure omits the global section that attends to and from all other tokens. Since column/row order can be pre-computed, the method is linear for a constant block size.}
    \label{fig:efficiency}
\end{figure*}

This work focuses on question answering (\qa) and entailment tasks on tables. While we apply our model to several such tasks (see section~\ref{sec:results}), \hqa~\cite{chen-etal-2020-hybridqa} is particularly interesting, as it requires processing tables jointly with long passages associated with entities mentioned
in the table, yielding large documents that may not fit in standard Transformer models.

Overall, our contributions are the following:

i) We show that table transformers naturally focus attention according to rows and columns, and that constraining attention to enforce this improves accuracy on three table reasoning tasks, yielding new state-of-the-art results in \sqa and \tabfact.

ii) We introduce \model, a novel transformer architecture that exploits table structure to allow running training and inference in longer sequences. Unlike traditional self-attention, \model scales linearly in the sequence length.

iii) We propose \pointr (Section~\ref{sec:pointr}), a novel two-phase framework that exploits \model to tackle large-scale \qa{} tasks, like \hqa, that require multi-hop reasoning over tabular and textual data. We improve the state-of-the-art by $19$ points.

All the code is available as open source.\footnote{\href{https://github.com/google-research/tapas}{github.com/google-research/tapas}}

%% file: sections/20-related.tex
\section{Related Work}
\label{sec:related}

\paragraph{Transformers for tabular data}
Traditionally, tasks involving tables were tackled  by searching for logical forms
in a semantic parsing setting. More recently Transformers~\cite{vaswani17} have been used
to train end-to-end models on tabular data as well~\cite{Chen2020TabFact}.
For example, \tapas~\cite{herzig-etal-2020-tapas} relies on Transformer-based masked language model pre-training and special row and column embeddings
to encode the table structure. \citet{chen2021open} 
use a variant of \etc ~\cite{ainslie-etal-2020-etc} 
on an open-domain version of \hqa
to read and choose an answer span from multiple candidate passages and cells, but the proposed model does not jointly process the full table with passages.

In order to overcome the limitations on sequence length~\citet{eisenschlos-etal-2020-understanding}
propose heuristic column selection techniques, and they also propose pre-training on synthetic data. 
\citet{krichene-etal-2021-dot} propose a model based cell selection technique that is differentiable and trained end-to-end together with the main task model.
Our approach is
orthogonal to these methods, and can be usefully combined with them, as shown in Table~\ref{tab:other_results}.

Recently, \citet{zhang-etal-2020-table}  proposed \sat, which uses an attention mask
to restrict attention to tokens in the same
row and same column. \sat also computes an additional histogram row appended at the bottom of the table and encodes the table content as text only (unlike \tapas). The proposed method is not head dependent
as ours is, which prevents it from being implemented efficiently to allow scaling to larger
sequence lengths. Controlling for model size and pre-training for a fair comparison, we show that our model is both faster (Table~\ref{fig:speed}) and more accurate (Table~\ref{tab:recall}) than \sat.

\paragraph{Efficient Transformers}
There is prior work that tries to improve the asymptotic
complexity of the self-attention mechanism in transformers.
\citet{tay2020efficient} review the different methods and cluster them based
on the nature of the approach. We cover some of the techniques below and show a theoretical complexity comparison in Table~\ref{tab:complexity}.

\input{tables/20-complexity}

The \linformer model~\citet{wang2020linformer} uses learned projections to reduce the sequence length
axis of the keys and value vectors to a fixed length. The projections are then
anchored to a specific input length which makes adapting the sequence length
during pre-training and fine-tuning challenging, and makes the model more sensitive to position offsets in sequences of input tokens.

\reformer ~\citep{kitaev2020reformer} uses locality sensitive hashing to reorder the input
tokens at every layer in such a way that similar contextual embeddings have
a higher chance of being clustered together. %
We instead rely on the input data structure to define ways to cluster the tokens.
Although the limitation can be circumvented by adapting the proposed architecture, \reformer was originally defined for
auto-regressive training.

~\citet{ainslie-etal-2020-etc} introduce \etc, a framework for global memory
and local sparse attention, and use the mechanism of relative positional attention~\cite{dai-etal-2019-transformer}
to encode hierarchy.
\etc was applied to large document tasks such as Natural Questions~\cite{kwiatkowski-19}.
The method does not allow different dynamic or static data re-ordering.
In practice, we have observed that the use of relative positional attention
introduces a large overhead during training. 
\bigbird~\cite{zaheer2020bigbird} presents a similar approach with the addition of attention to random tokens. 

%% file: tables/20-complexity.tex
\begin{table}[t]
    \centering
    \begin{tabular}{lccl}
    \hline
       Model  &  Complexity & Class\\
       \hline
         Transformer-XL
         & $\mathcal{O}(n^2)$ & RC \\ 
        \reformer 
        & $\mathcal{O}(n \log n)$ & LP\\
           \linformer 
           & $\mathcal{O}(n)$ & LR \\
        \bigbird 
        & $\mathcal{O}(ng + nb)$ & FP+RP+M\\ 
        \etc 
        & $\mathcal{O}(ng + nb)$ & FP+M\\ 
        \model (Ours) & $\mathcal{O}(ng + nb)$ & FP\\ 
        \hline
    \end{tabular}
    \caption{Comparison of transformer models following~\citet{tay2020efficient}. Class abbreviations include: FP = Fixed Patterns, RP = Random Patterns, M = Memory, LP = Learnable Pattern, LR = Low Rank and RC = Recurrence. The block size for local attention is denoted as $b$ and $g$ the size of the global memory. For our \model model, a $n \log n$ sorting step can be pre-computed before training or inference for known tables so it is omitted.}
    \label{tab:complexity}
\end{table}

%% file: sections/30-model.tex
\section{The \model model}
\label{sec:models}

Following \abr{tapas}~\cite{herzig-etal-2020-tapas}, the transformer input in \model, for each table-QA example, is the query and the table, tokenized and flattened, separated by a \texttt{[SEP]} token, and prefixed by a \texttt{[CLS]}. Generally the table comprises most of the the input. We use the same row, column and rank embeddings as \tapas.

To restrict attention between the tokens in the table, we propose having some attention heads limited to attending between tokens in the same row (plus the non-table tokens), and likewise for columns.
We call these \emph{row heads} and \emph{column heads} respectively.
In both cases, we allow attention to and from all the non-table tokens.

Formally, if $\mathbf X\in\mathbb R^{d\times n}$ is the input tensor for a Transformer layer with sequence length $n$, the $k$-th position of the output of the $i$-th attention head is:
\vspace{-1mm}
\[
\text{Head}^i_k\left(\mathbf X\right) = \mathbf W^i_V \mathbf{X}_{\mathcal{A}^i_k}
\sigma\left[
\left(\mathbf{W}^i_K \mathbf{X}_{\mathcal{A}^i_k}\right)^\intercal\mkern-2mu
\mathbf W^i_Q \mathbf X_k
\right]
\vspace{-1mm}
\]
where $\mathbf W^i_Q, \mathbf W^i_K, \mathbf W^i_V \in \mathbb R^{m\times d}$ are query, key and value projections respectively, $\sigma$ is the \texttt{softmax} operator, and $\mathcal{A}^i_k \subseteq \left\{1,\cdots, n\right\}$ represents the set of tokens that position $k$ can attend to, also known as the \emph{attention pattern}. 
Here $\mathbf X_{\mathcal{A}^i_k}$ denotes gathering from $\mathbf X$ only the indexes in $\mathcal{A}^i_k$.
When $\mathcal{A}^i_k$ contains all positions (except padding) for all heads $i$ and token index $k$ then we are in the standard dense transformer case. For a token position $k$, we define $r_k, c_k \in \mathbb N_0$ the row and column number, which is set to $0$ if $k$ belongs to the query set $Q$: the set of token positions in the query text including \texttt{[CLS]} and \texttt{[SEP]} tokens.

In \model, we use two types of attention patterns. The first $h_r\geq 0$ heads are \emph{row heads} and the remaining $h_c$ are \emph{column heads}:
\vspace{-2.5mm}
\[
\mathcal{A}^i_k = \begin{cases}
\left\{1,\cdots,n\right\} &~\text{if}~ k \in Q\text{, else} \\
Q \cup \left\{j : r_j = r_k\right\} &~\text{if}~ 1 \leq i \leq h_r \\
Q \cup \left\{j : c_j = c_k\right\} &~\text{otherwise.}
\end{cases}
\]

One possible implementation of this is an attention mask that selectively sets elements in the attention matrix to zero. (Similar masks are used for padding tokens, or auto-regressive text generation.) 
The ratio of row and column
heads is a hyper-parameter but empirically we found a $1:1$ ratio to work well.
In Section~\ref{sec:results}
we show that attention masking improves accuracy on four table-related tasks.
We attribute these improvements to better inductive bias, and support this in Section~\ref{sec:analysis} showing that full attention models learn to approximate this behavior.

\subsection{Efficient implementation}

Although row- and column-related attention masking improves accuracy, it does not improve Transformer efficiency---despite the restricted attention, the Transformer still uses quadratic memory and time.  We thus also present an approximation of row and column heads that can be implemented more efficiently.
Inspired by~\citet{ainslie-etal-2020-etc}, the idea is to divide the input into
a global part of length $G$ that attends to and from everything, and a local (typically longer) part that attends only to the global section and some radius $R$ around each token in the  sequence. 
\etc does this based on a fixed token order.
However, the key insight used in \model is that the notion of locality can be configured \textit{differently for each head}: one does not need to choose a specific traversal order for tokens ahead of time, but instead tokens can be ordered in a data-dependent (but deterministic) way.
In particular, row heads can order the input according to a row order traversal of the table, and column heads can use a column order traversal.
The architecture is shown in Figure~\ref{fig:efficiency}.

After each head has ordered its input we split off the first $G$ tokens and group the rest in evenly sized buckets of length $R$.
By reshaping the input matrix in the self-attention layer to have $R$ as the last dimension,
one can compute attention scores from each bucket to itself, or similarly from each bucket to an adjacent one.
Attention is further restricted with a mask to ensure row heads and column heads don't attend across rows and columns respectively. See model implementation details in Appendix~\ref{sec:apx-code}.
When $G$ is large enough to contain the question part of the input and $R$ is large enough to fit an entire column or row, then the efficient implementation matches the mask-based one.

As observed in ~\citet{ainslie-etal-2020-etc},  asymptotic complexity improvements
often do not materialize for small sequence lengths, given the overhead of
tensor reshaping and reordering. The exact break-even point will depend on several factors,
including accelerator type and size as well as batch size. 
In the experiments below the best of the two functionally equivalent implementations of \model is chosen for each use case.

\subsection{Compatibility with \bert weights}

The sparse attention mechanism of \model 
adds no additional parameters. As a consequence, a \model checkpoint is compatible with any \bert or \tapas pre-trained checkpoint. Following~\citet{herzig-etal-2020-tapas} we obtained best results running the same masked language model pre-training used in \tapas with the same data but using the sparse attention mask of \model.

For sequence lengths longer than $512$ tokens, we reset the index of the positional embeddings at the beginning of each cell.
This method removes the need to learn positional embeddings for larger indexes as the maximum sequence length grows while avoiding the large computational cost of relative positional embeddings.

\subsection{Universal approximators}

\citet{Yun2020Are} showed that Transformers are universal approximators for any continuous sequence-to-sequence function, given sufficient layers.
This result was further extended by ~\citet{univapprox2020, zaheer2020bigbird} to some Sparse Transformers under reasonable assumptions. 

However, prior work limits itself to the case of a single attention pattern per layer, whereas \model uses different attention patterns depending on the head. We will show that \model is also a universal approximator for sequence to sequence functions.

Formally, let $\mathcal{F}$ be the class of continuous functions $f:\mathbb{D}\subset \mathbb{R}^{d\times n}\to \mathbb{R}^{d\times n}$
with $\mathbb{D}$ compact, with the $p$-norm $|\!|\cdot|\!|_p$. Let $\mathcal{T}_\model$ be any family of transformer models with a fixed set of hyper-parameters (number of heads, hidden dimension, etc.) but with an arbitrary number of layers. Then we have the following result.

\begin{theorem}
If the number of heads is at least $3$ and the hidden size of the feed forward layer is at least $4$, then for any $f\in \mathcal{F}$ and 
$\epsilon \in \mathbb{R}_+$ 
there exists $\hat{f}\in\mathcal{T}_\model$ such that $|\!|\hat{f} - {f}|\!|_p < \epsilon$.
\label{theo:universal}
\end{theorem}

See the Appendix~\ref{sec:apx-universal} for a detailed proof, which relies on the fact that $3$ heads will guarantee at least two heads of the same type. The problem can then be reduced to the results of~\citet{univapprox2020}.

\section{The \pointr architecture}
\label{sec:pointr}
\input{tables/41-hybridqa_stats}

Many standard table QA datasets~\cite{pasupat2015compositional, Chen2020TabFact, iyyer-etal-2017-search}, perhaps by design, use tables that can be limited to $512$ tokens.
Recently, more datasets~\cite{kardas-etal-2020-axcell, talmor2021multimodalqa} requiring parsing larger semi-structured documents have been released.

\input{figures/hybridqa_example}

Among them, we focus on \hqa~\cite{chen-etal-2020-hybridqa}. 
It uses Wikipedia tables with entity links, with answers taken from either a cell or a hyperlinked paragraph.
Dataset statistics are shown in Table~\ref{tab:hybridqa_data_split}. Each question contains a table with on average $70$ cells and $44$ linked entities. Each entity is represented by the first $12$ sentences of the Wikipedia description, averaging $100$ tokens. The answer is often a span extracted from the table or paragraphs but the dataset has no ground truth annotations on how the span was obtained, leaving around $50\%$ of ambiguous examples where more than one answer-sources are possible.
The total required number of word pieces accounting for the table, question and entity descriptions grows to more than $11,000$ if one intends to cover more than $90\%$ of the examples, going well beyond the limit of traditional transformers.

To apply sparse transformers to the \hqa task, we propose \pointr, a two stage framework in a somewhat similar
fashion to open domain QA pipelines~\cite{chen-etal-2017-reading, lee-19}.
We \emph{expand} the cell content by appending the descriptions of its linked entities.
The two stages of \pointr correspond to (Point)ing to the correct \emph{expanded cell} and then (R)eading a span from it. 
See Figure~\ref{fig:hybridqa_example} for an example. Full set-up details are discussed in Appendix~\ref{sec:apx-exp}.

\input{sections/41-cell-selection}
\input{sections/42-reader}

%% file: tables/41-hybridqa_stats.tex
\begin{table}[t]
\small
\centering
\begin{tabular}{lrrrr}
\toprule
Split & Train & Dev & Test & Total \\
\midrule
In-Passage  &  35,215 & 2,025  & 20,45  & 39,285 %
\\
In-Table & 26,803 & 1,349 & 1,346 & 29,498 %
\\
Missing & 664  & 92 & 72 & 828 %
\\
Total & 62,682 & 3,466 & 3,463 & 69,611 \\
\bottomrule
\end{tabular}
\caption{Statistics for \hqa. \emph{In-Table} and \emph{In-Passage} groups mark the location of the answer. Missing denotes answers that do not match any span and may require complex computations.}
\label{tab:hybridqa_data_split}
\end{table}

%% file: figures/hybridqa_example.tex
\begin{figure*}[th]
\begin{tabular}{p{14cm}}
\textbf{Original Input Table:} \\
\end{tabular}
\centering
\small
\begin{tabular}{llllll}
\textbf{Pos} &	\textbf{No} & \textbf{Driver} & \textbf{Constructor} & \textbf{Time} & \textbf{Gap} \\
\hline
1	&2	&\underline{Rubens Barrichello}	&\underline{Ferrari} 	&1:10.223	- \\
2	&1	&\underline{Michael Schumacher}	&\underline{Ferrari}	&1:10.400	&+0.177 \\
3	&10	&\underline{Takuma Sato}	&\underline{Honda}	&1:10.601	&+0.378 \\
4	&9	&\underline{Jenson Button}~$\star$	&\underline{Honda}	&1:10.820	&+0.597 \\
\end{tabular}%
\\
\vspace{6pt}
\begin{tabular}{p{14cm}}
\textbf{Entity Descriptions (some are omitted for space):} \\

$\bullet$~\emph{Rubens Barrichello} is a Brazilian racing driver who competed in Formula One between 1993 and 2011. \\
$\bullet$~\emph{Jenson Alexander Lyons Button} MBE is a British racing driver. He won the 2009 F1 World Championship. \\
$\bullet$~\emph{Scuderia Ferrari S.p.A.} is the racing division of luxury Italian auto manufacturer Ferrari. Ferrari supplied cars complete with V8 engines for the A1 Grand Prix series from the 2004 season. \\
$\bullet$~\emph{Honda Motor Company, Ltd} is a Japanese public multinational conglomerate manufacturer of automobiles, motorcycles, and power equipment, headquartered in Minato, Tokyo, Japan. \\
\hspace{7cm}\vdots \\
\textbf{Question:} \\
The driver who finished in position 4 in the 2004 Grand Prix was of what nationality? \textbf{British} \\
\vspace{-1mm}
\hspace{7cm}$\big\Downarrow$ \\
\vspace{-1mm}
\textbf{Expanded Table with $k$ Description Sentences Most Similar to the Question:} \\
\resizebox{14cm}{!}{
\begin{tabular}{llp{6cm}p{7cm}ll}
\textbf{Pos} &	\textbf{No} & \textbf{Driver} & \textbf{Constructor} & \textbf{Time} & \textbf{Gap} \\
\hline
1	&2	& Rubens Barrichello	& Ferrari (Ferrari supplied cars complete with V8 engines for the A1 Grand Prix series from the 2004 season.) 	&1:10.223	- \\
2	&1	& Michael Schumacher	& Ferrari (Ferrari supplied cars complete with V8 engines for the A1 Grand Prix series from the 2004 season.)	&1:10.400	&+0.177 \\
3	&10	& Takuma Sato	& Honda	& 1:10.601	&+0.378 \\
4	&9	& Jenson Button (Jenson Alexander Lyons Button MBE is a \textbf{British} racing driver.)~$\star$	& Honda	&1:10.820	&+0.597 \\
\end{tabular}}
\\
\vspace{5pt}
\textbf{\pointr inference pipeline:} \\
\centering
\includegraphics[width=.85\linewidth]{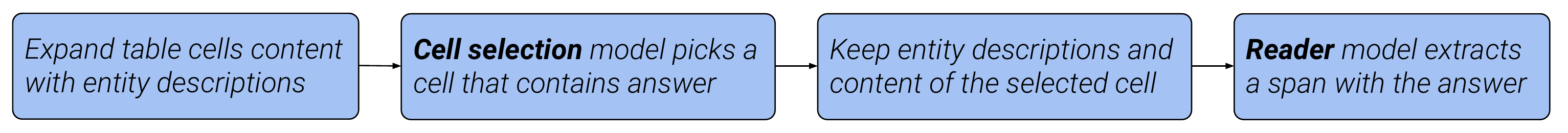}
\end{tabular}%
\caption{An example from the \hqa dataset processed by \pointr. The first paragraph in the Wikipedia page for each underlined entity was available to the dataset authors. 
We expand the text in the cells with this descriptions for the top-$k$ most relevant sentences, as shown in the second table, and train a model to find the cell containing or linking to the answer (marked here with a $\star$). The goal is to provide the model with all the context needed to locate the answer. A second model extracts a span from the selected cell content and linked text.}
\label{fig:hybridqa_example}
\end{figure*}

%% file: sections/41-cell-selection.tex
\subsection{\pointr: Cell Selection Stage}

In the first stage we train a cell selection model using \model whose objective is to select the expanded cell that contains the answer. 
\model accepts the full table as input;
therefore, expanding all the cells with their respective passages is impractical.
Instead, we consider the top-$k$ sentences in the entity descriptions for expansion, using a \abr{tf-idf}
metric against the query. Using $k=5$, we can fit $97\%$ of the examples in $2048$ tokens; for the remaining examples, we truncate the longest cells uniformly until they fit in the budget. 

The logit score $S$ for each cell $c$ is obtained by mean-pooling the logits for each of the tokens $t$ inside it, which are in turn the result of applying a single linear layer to the contextual representation of each token when applying \model to the query $q$ and the expanded table $e$.
\begin{align*}
S(t)&=\texttt{MLP}(\texttt{MATE}(q, e)[t]) \\
S(c)&=\texttt{avg}_{t\in c} S(t) \\
P(c)&=\frac{\texttt{exp}(S(c))}
          {\sum_{c'\in e} \texttt{exp}(S(c'))}
\end{align*}
We use cross entropy loss for training the model to select expanded cells that contain the answer span.
Even though the correct span may appear in multiple cells or passages, in practice many of these do so only by chance and do not correspond to a reasoning path consistent with the question asked. 
In Figure~\ref{fig:hybridqa_example} for instance, there could be other British divers but we are only interested in selecting the cell marked with a \emph{star} symbol ($\star$).
In order to handle these cases we rely on \emph{Maximum Marginal Likelihood} (\abr{mml})~\cite{liang-etal-2013-learning, berant-etal-2013-semantic}. As shown by ~\citet{guu-etal-2017-language} \abr{mml} can be interpreted as using the online model predictions (without gradients) to compute a soft label distribution over candidates. 
For an input query $x$, and a set $\mathcal C$ of candidate cells, the loss is:
\[
\mathcal L(\Theta, x, \mathcal{C}) = \sum_{z\in\mathcal C} - q(z) \log p_\Theta(z|x)
\]
with $q(z) = p_\Theta(z|x, z\in\mathcal{C})$ the probability distribution given by the model restricted to candidate cells containing the answer span, taken here as a constant with zero gradient.

%% file: sections/42-reader.tex
\subsection{\pointr: Passage Reading Stage}
In the second stage we develop a span selection model that reads the answer from a single \emph{expanded} cell selected by the \pointr Cell Selector.
In order to construct the expanded cell for each example, we concatenate the cell content with all the sentences of the linked entities
and keep the first $512$ tokens.

Following various recent neural machine reading works~\cite{chen-etal-2017-reading, lee-19, herzig-etal-2021-open}, we fine-tune a pre-trained \bert-uncased-large model~\cite{devlin-19} that attempts to predict a text span from the text in a given table cell $c$ (and its linked paragraphs) and the input query $q$. 
We compute a span representation as the concatenation of the contextual embeddings of the first and
last token in a span $s$ and score it using a multi-layer perceptron:
\begin{align*}
    h_{start} &= \texttt{BERT}_r(q, c)[\texttt{START}(s)] \\
    h_{end} &= \texttt{BERT}_r(q, c)[\texttt{END}(s)] \\
    \texttt{S}_\text{read}(q, c) &= \texttt{MLP}([h_{start}, h_{end}])
\end{align*}

A softmax is computed over valid spans in the input and the model is trained with cross entropy loss. If the span-text appears multiple times in a cell we consider only the first appearance. 
To compute EM and F1 scores during inference, we evaluate the trained reader on the highest ranked cell output predictions of the \pointr Cell Selector
using the official evaluation script.

%% file: sections/40-experiments.tex
\section{Experimental Setup}
\label{sec:experiments}

We begin by comparing the performance of \model on \hqa to other existing systems. We focus on prior efficient transformers to compare the benefits of the table-specific sparsity. We follow ~\citet{herzig-etal-2020-tapas, eisenschlos-etal-2020-understanding} in reporting error bars with the interquartile range.

\input{tables/40-dataset_stats}

\input{sections/43-baselines}

\subsection{Other datasets}

We also apply \model to three other datasets involving tables to demonstrate that the sparse attention bias yields stronger table reasoning models.
\sqa \cite{iyyer-etal-2017-search} is a sequential QA task, \wtq~\cite{pasupat2015compositional} is a QA task that sometimes also requires aggregation of table cells, and \tabfact~\cite{Chen2020TabFact} is a binary entailment task. See Table~\ref{tab:dataset_stats} for dataset statistics. 
We evaluate with and without using the intermediate pre-training tasks (CS)~\cite{eisenschlos-etal-2020-understanding}.

%% file: tables/40-dataset_stats.tex
\begin{table}[t]
\begin{center}
\scalebox{0.85}{
\begin{tabular}{lrrr}
\toprule
 & \tabfact & \wtq & \sqa \\
\midrule
Examples         & 118,275      & 22,033  & 17,553  \\
Tables           & 16,573      & 2,108   & 982 \\ 
\bottomrule
\end{tabular}
}
\end{center}
\caption{Statistics for \sqa, \wtq and \tabfact.}
\label{tab:dataset_stats}
\end{table}

%% file: sections/43-baselines.tex
\subsection{Baselines}

The first baselines for \hqa are Table-Only and Passage-Only, as defined in ~\citet{chen-etal-2020-hybridqa}. 
Each uses only the part of the input indicated in the name but not both at the same time. 
Next, the \hybrider model from the same authors, consists of four stages: entity linking, cell ranking, cell hopping and finally a reading comprehension stage, equivalent to our final stage. 
The first three stages are equivalent to our single cell selection stage; hence, we use their reported error rates to estimate the retrieval rate. 
The simpler approach enabled by \model avoids error propagation and yields improved results. 

We also consider two recent efficient transformer architectures as alternatives for the \pointr Cell Selector,
one based on \linformer~\cite{wang2020linformer} and one based on \etc~\cite{ainslie-etal-2020-etc}. In both cases we preserve the row, column and rank embeddings introduced
by ~\citet{herzig-etal-2020-tapas}.
\linformer learns a projection matrix that reduces the sequence length dimension of the keys and values tensor to a fixed length of $256$ (which performed better than $128$ and $512$ in our tests.)
\etc is a general architecture which requires some choices to be made about how to allocate global memory and local attention~\cite{dai-etal-2019-transformer}
Here we use a $256$-sized global memory to summarize the content of each cell, by assigning each token in the first half of the global memory to a row,
and each token in the second half to a column.
Tokens in the input use a special
relative positional value to mark when they are interacting with their corresponding
global row or column memory position. We will refer to this model as \tableetc.

Finally we consider two non-efficient models: A simple \tapas model without any sparse mask, and an \sat~\cite{zhang-etal-2020-table} model pretrained on the same MLM task as \tapas for a fair comparison. For the cell selection task \tapas obtains similar results to \model, but both \tapas and \sat lack the efficiency improvements of \model. 

%% file: sections/50-results.tex
\section{Results}
\label{sec:results}

In Figure~\ref{fig:speed} we compare inference speed of different models as we increase the sequence length. 
Similar results showing number of FLOPS and memory usage are in Appendix~\ref{sec:apx-exp}.
The linear scaling of \linformer and the linear-time version \model can be seen clearly.
Although \linformer has a slightly smaller linear constant, the pre-training is 6 times slower, as unlike the other models, \linformer pretraining must be done at the final sequence length of $2048$.

\begin{figure}
    \centering
    \includegraphics[width=.95\linewidth]{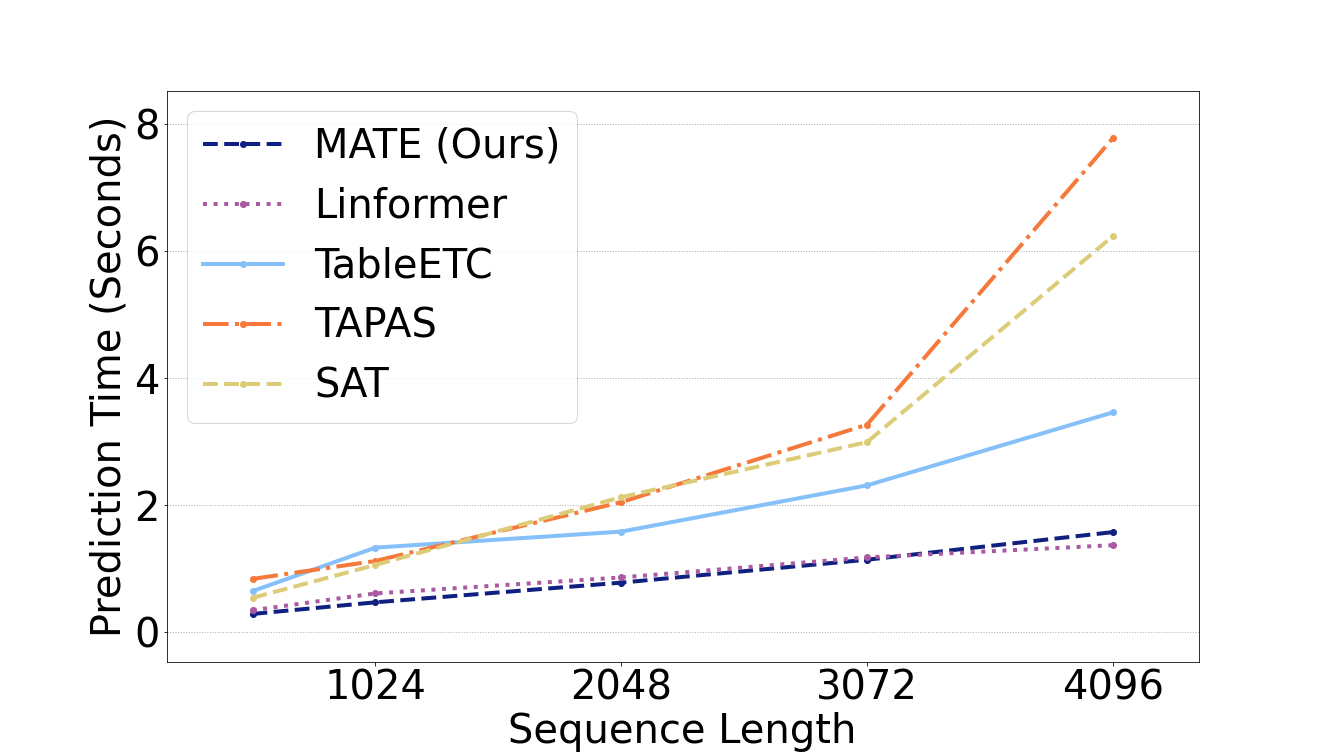}
    \caption{Comparison of inference speed on a cloud VM with 64GB. At a sequence length of 2048, \model is nearly twice as fast as \tapas.}
    \label{fig:speed}
\end{figure}

\input{tables/50-table_parsing}

Table~\ref{tab:results} shows the end-to-end results of our system using \pointr with \model on \hqa, compared to the previous state-of-the-art as well as the other efficient transformer baselines from  Section ~\ref{sec:experiments}.
\model outperforms the previous SOTA \hybrider by over 19 points, and \linformer, the next best efficient-transformer system, by over 2.5 points, for both exact-match accuracy and F1.

\input{tables/52-hybridqa_results}

We also applied \model to three tasks involving table reasoning over shorter sequences. 
In Table~\ref{tab:other_results} we see that \model provides improvements in accuracy, which we attribute to a better inductive bias for tabular data. When combining \model with \emph{Counterfactual + Synthetic} intermediate pre-training (CS)~\cite{eisenschlos-etal-2020-understanding} we often get even better results. For \tabfact and \sqa we improve over the previous state-of-the-art. For \wtq we close the gap with the best published system \tabert~\cite{yin-etal-2020-tabert} (51.8 mean test accuracy), which relies on traditional semantic parsing, instead of an end-to-end approach. Dev results show a similar trend and can be found in Appendix~\ref{sec:apx-dev-results}. No special tuning was done on these models---we used the same hyper-parameters as the open source release of \tapas.

%% file: tables/50-table_parsing.tex
\begin{table}[t]
\small
\centering
\resizebox{1\columnwidth}{!}{
\begin{tabular}{l|cccc}
\toprule
Model & \sqa \abr{all} & \sqa \abr{seq} & \wtq  & \tabfact \\
\midrule
\tapas & 67.2\err{0.5} & 40.4\err{0.9} & 42.6\err{0.8} & 76.3\err{0.2} \\
\model & \textbf{71.6}\err{0.1} & \textbf{46.4}\err{0.3} & \textbf{42.8}\err{0.8} & \textbf{77.0}\err{0.3} \\
\midrule
\tapas + CS & 71.0\err{0.4} & 44.8\err{0.8} & 46.6\err{0.3} & 81.0\err{0.1} \\
\model + CS & \textbf{71.7}\err{0.4} & \textbf{46.1}\err{0.4} & \textbf{51.5}\err{0.2} & \textbf{81.4}\err{0.1} \\
\bottomrule
\end{tabular}
}
\caption{Test results of using \model on other table parsing datasets show improvements due to the sparse attention mechanism. Using \emph{Counterfactual + Synthethic} pretraining (CS) in combination with \model achieves state-of-the-art in \sqa and \tabfact. Errors are estimated with half the interquartile range over 5 runs.}
\label{tab:other_results}
\end{table}

%% file: tables/52-hybridqa_results.tex
\begin{table*}[!ht]
\small
\centering
\resizebox{1.0\textwidth}{!}{
    \begin{tabular}{l|cccccc|cccccc}
    \toprule
    \multicolumn{1}{l|}{\multirow{1}{*}{Model}} & \multicolumn{6}{c|}{Dev} & \multicolumn{6}{c}{Test}    \\ 
    \multicolumn{1}{l|}{}                      & \multicolumn{2}{c}{In-Table} & \multicolumn{2}{c}{In-Passage} & \multicolumn{2}{c|}{Total} & \multicolumn{2}{c}{In-Table} & \multicolumn{2}{c}{In-Passage} & \multicolumn{2}{c}{Total}   \\
                               & EM & F1   & EM & F1 & EM & F1    & EM & F1       & EM & F1 & EM & F1   \\ 
    \midrule
    Table-Only & 14.7&19.1 & 2.4&4.5 & 8.4&12.1 & 14.2&18.8 & 2.6&4.7 &  8.3&11.7 \\
    Passage-Only & 9.2&13.5 & 26.1&32.4 & 19.5&25.1 & 8.9&13.8 & 25.5&32.0 & 19.1&25.0 \\
    \hybrider ($\tau$=0.8)     & 54.3&61.4  &  39.1&45.7  &  44.0&50.7  & 56.2&63.3 & 37.5&44.4  & 43.8&50.6   \\
    \midrule
    \pointr + \sat      & 66.5\err{0.33} & 71.8\err{0.28} & 60.3\err{0.11} & 69.2\err{0.04} & 61.2\err{0.29} & 68.7\err{0.31}  & 64.6 & 70.1   & 59.6 & 68.5   & 60.1 & 67.4 \\ %
    \pointr + \tapas        & 68.1\err{0.33} & 73.9\err{0.37} & 62.9\err{0.25} & 72.0\err{0.21} & \textbf{63.3}\err{0.25} & \textbf{70.8}\err{0.12}  & 67.8 & 73.2   & 62.0 & 70.9   & 62.7 & 70.0\\  %
    \midrule
    \pointr +\tableetc       & 36.0\err{1.26} & 42.4\err{1.13} & 37.8\err{1.19} & 45.3\err{1.53} & 36.1\err{1.30} & 42.9\err{1.36} & 35.8 & 40.7   & 38.8 & 45.7   & 36.6 & 42.6 \\
    \pointr + \linformer      & 65.5\err{0.78} & 71.1\err{0.55} & 59.4\err{0.59} & 69.0\err{0.68} & 60.8\err{0.68} & 68.4\err{0.63}  & 66.1 & 71.7   & 58.9 & 67.8   & 60.2 & 67.6 \\ %
    \pointr + \model   & 68.6\err{0.37} & 74.2\err{0.26} & 62.8\err{0.25} & 71.9\err{0.20} & \textbf{63.4\err{0.16}} & \textbf{71.0\err{0.17}}  & 66.9 & 72.3  & 62.8 & 71.9 & \textbf{62.8}&\textbf{70.2}\\ %
    
    \midrule
    Human       &  &   & & & & & & & & & 88.2&93.5 \\
    \bottomrule
    \end{tabular}
}
\caption{Results of different large transformer models on \hqa. In-Table and In-Passage subsets refer to the location of the answer. For dev, we report errors over 5 runs using half the interquartile range. Since the test set is hidden and hosted online, we report the results corresponding to the model with the median total EM score on dev.
}
\label{tab:results}
\end{table*}

%% file: sections/60-analysis.tex
\section{Analysis}
\label{sec:analysis}
\paragraph{\hqa Error analysis} We randomly sample 100 incorrectly answered examples from the development set. 55\% of the examples have lexical near-misses---predictions have the correct information, but have slightly different formatting (e.g. (Q)uestion: \textit{In what round was the Oklahoma athlete drafted in?} (G)old answer: ``second'', (P)redicted: ``second round''). 
While around 30\% of such misses involved numerical answers (eg: ``1'' vs ``one''), the predictions for the rest of them prominently (58\% of the near misses) either had redundant or were missing auxiliary words (e.g., Q: \textit{What climate is the northern part of the home country of Tommy Douglas?} G: ``Arctic'' P: ``Arctic climate''). The inconsistency in the gold-answer format and unavailability of multiple gold answers are potential causes here.

Among the non near-misses, the majority predictions were either numerically incorrect, or were referencing an incorrect entity but still in an relevant context---especially the questions involving more than 2 hops.
(e.g. Q: \textit{In which sport has an award been given every three years since the first tournament held in 1948-1949?} G: ``\entity{Badminton}'', P: ``\entity{Thomas Cup}''). 
Reassuringly, for a huge majority (> 80\%), the entity type of the predicted answer (person, date, place, etc.) matches the type of the gold answer.
The observed errors suggest potential gains by improving the entity~\cite{Xiong2020Pretrained} and numerical~\cite{andor-etal-2019-giving} reasoning skills.

\input{tables/51-retrieval}

\paragraph{Ablation Study} In Table~\ref{tab:recall} we compare architectures for cell selection on \hqa. Hits@k corresponds to whether a cell containing an answer span was among the top-k retrieved candidates. As an ablation, we remove the sparse pre-training and try using only row/column heads. We observe a drop also when we discard the ambiguous examples from training instead of having \abr{mml} to deal with them. Unlike the other datasets, \tapas shows comparable results to \model, but without any of the theoretical and practical improvements.

\paragraph{Observed Attention Sparsity}
Since we are interested to motivate our choices on how to sparsify the attention matrix, we can
inspect the magnitude of attention connections in a trained dense \tapas model
for table question answering.
It is important to note that in this context we are not measuring attention as an explanation method \cite{jain-wallace-2019-attention, wiegreffe-pinter-2019-attention}.
Instead we are treating the attention matrix in the fashion of magnitude based
pruning techniques \cite{hanpruning, see-etal-2016-compression}, and simply consider between which pairs of tokens the scores are concentrated.

Given a token in the input
we can aggregate the attention weights flowing from it depending on the position
of the target token in the input (\texttt{CLS} token, question, header, or table) and whether
the source and target tokens are in the same column or row, whenever it makes sense.
We average scores across all tokens, heads, layers and examples in the development set.
As a baseline, we also compare against the output of the same process when using
a uniform attention matrix, discarding padding.

In Table~\ref{tab:attention}, we show the obtained statistics considering only
table tokens as a source. We use the \wtq development set as a reference. While we see that $23\%$ of the attention weights are looking at
tokens in different columns and rows, this is only about one third of the baseline
number one would obtain with a uniform attention matrix. This effect corroborates the approach taken in \model.

\begin{table}
\small
\centering
\resizebox{\columnwidth}{!}{
\begin{tabular}{lccrrr}
\toprule
Type            & \multicolumn{1}{p{0.95cm}}{Same column} & \multicolumn{1}{p{0.5cm}}{Same row} & \multicolumn{1}{p{1.3cm}}{Attention / Uniform} & Attention \% & Uniform \% \\
\midrule
\texttt{[CLS]}          &  &  & 44.46 & 12.45  &  0.28 \\
\midrule
Question                &  &  & 10.13 & 37.08  &  3.66 \\
\midrule
\multirow{2}{*}{Header} & \xmark &  & 1.32 &  3.89  &  2.94 \\
                        & \cmark &  &  3.37 & 1.72 & 0.51 \\
\midrule
\multirow{4}{*}{Table}  & \xmark & \xmark & \textbf{0.34} & 23.59 & 68.18 \\
                        & \xmark & \cmark & 0.87 & 3.51 &  4.02 \\
                        & \cmark & \xmark & 0.70 & 13.16  & 18.84 \\
                        & \cmark & \cmark & 2.93 & 4.60  &  1.57 \\
\bottomrule
\end{tabular}
}
\caption{Average attention flow from a token in the table to other token types.
We compare to an uniform attention matrix as a baseline.
Attention to tokens in different rows and columns is the relative smallest with one third of the baseline. Computed on \wtq Dev.}
\label{tab:attention}
\end{table}

%% file: tables/51-retrieval.tex
\begin{table}[!t]
\small
\centering
\resizebox{1\columnwidth}{!}{
\begin{tabular}{l|ccc}
\toprule
Model & Hits@1 & Hits@3  & Hits@5 \\
\midrule
\hybrider & 68.5 & - & -  \\
\midrule
\sat  & 77.9 & 87.4 & 90.3 \\
\tapas & \textbf{80.1} & \textbf{89.5} & 91.4 \\
\midrule
\tableetc & 51.1 & 72.0 & 78.9  \\
\linformer & 77.1 & 86.5 & 90.0  \\
\model & \textbf{80.1} & 89.2 & \textbf{91.5} \\
\midrule
\model (-- row heads) & 78.3 & 87.7 & 90.3  \\
\model (-- col heads) & 77.8 & 87.1 & 90.0  \\
\model (-- sparse pretrain) & 75.5 & 86.5 & 89.9  \\
\model (-- ambiguous) & 76.7 & 84.2 & 86.6  \\
\bottomrule
\end{tabular}
}
\caption{Retrieval results over \hqa (dev set) for models used in \pointr Cell Selection stage. 
Efficient transformer models are grouped together.
\hybrider results are obtained from \citet{chen-etal-2020-hybridqa} by composing the errors for the first components.}
\label{tab:recall}
\end{table}

%% file: sections/70-conclusion.tex
\section{Conclusion}
\label{sec:conclusion}

We introduce \model, a novel method for efficiently restricting the attention flow in Transformers applied to Tabular data.
We show in both theory and practice that the method improves inductive bias and
allows scaling training to larger sequence lengths as a result of linear complexity.
We improve the state-of-the-art on \tabfact, \sqa and \hqa, the last one by $19$ points.

\section*{Ethical Considerations}

Although one outcome of this research is more efficient Transformers for table data, it remains true that large Transformer models can be expensive to train from scratch, so experiments of this sort can incur high monetary cost and carbon emissions. This cost was reduced by conducting some experiments at relatively smaller scale, e.g.~the results of Figure~\ref{fig:speed}.  To further attenuate the impact of this work, we plan release all the models that we trained so that other researchers can reproduce and extend our work without re-training.

All human annotations required for the error analysis (Section~\ref{sec:analysis}) are provided by authors, and hence a concern of fair compensation for annotators did not arise.

\section*{Acknowledgments}
We would like to thank Yasemin Altun, Ankur Parikh, Jordan Boyd-Graber, Xavier Garcia, Syrine Krichene, Slav Petrov, and
the anonymous reviewers for their time, constructive
feedback, useful comments and suggestions about this work.

%% file: sections/99-appendix.tex
\appendix\section*{Appendix}

We provide all details on our experimental setup to reproduce the results in Section~\ref{sec:apx-exp}.
In Section~\ref{sec:apx-dev-results} we show the development set results for our experiments. 
The proof for Theorem~\ref{theo:universal} is described in Section~\ref{sec:apx-universal} and in Section~\ref{sec:apx-code} we include the main code blocks for implementing \model efficiently on a deep learning framework.

\section{Experimental setup}
\label{sec:apx-exp}

\subsection{Pre-training}

Pre-training for \model was performed with constrained attention with a masked language modeling objective applied to the corpus of tables and text extracted by ~\citet{herzig-etal-2020-tapas}. 
With a sequence length of $128$ and batch size of $512$, the total training of $1$ million steps took $2$ days.

In contrast, for \linformer the pre-training was done with a sequence length of $2048$ and a batch size of $128$, and the total training took $12$ days for $2$ million steps. For \tableetc we also pre-trained for $2$ million steps but the batch size had to be lowered to $32$.
In all cases the hardware used was a 32 core \textbf{Cloud TPUs V3}.

\subsection{Fine-tuning}

For all experiments we use Large models over $5$ random seeds and report the median results. Errors are estimated with half the interquartile range. For \tabfact, \sqa and \wtq we keep the original hyper-parameters used in \tapas and provided in the open source release. In Figure~\ref{fig:flops} we show the floating point operation count of the different Transformer models as we increase the sequence length, as extracted from the execution graph. We also measure the memory doing CPU inference in figure
~\ref{fig:memory}. The linear scaling of \linformer and \model can be observed. 
No additional tuning or sweep was done to obtain the published results. We set the global size $G$ to $116$ and the radius $R$ for local attention to $42$. We use an Adam optimizer with weight decay with the same configuration as \bert.
The number of parameters for \model is the same as for \bert: $340M$ for Large models and $110M$ for Base Models.

In the \hqa cell selection stage, we use a batch size of $128$ and train for $80,000$ steps and a sequence length of $2048$. Training requires $1$ day.
We clip the gradients to $10$ and use a learning rate of $1\times 10^{-5}$ under a $5\%$ warm-up schedule.
For the reader stage use a learning rate of $5\times 10^{-5}$ under a $1\%$ warm-up schedule, a batch size of $512$ and train for $25,000$ steps, which takes around $6$ hours.

\begin{figure}
    \centering
    \includegraphics[width=.95\linewidth]{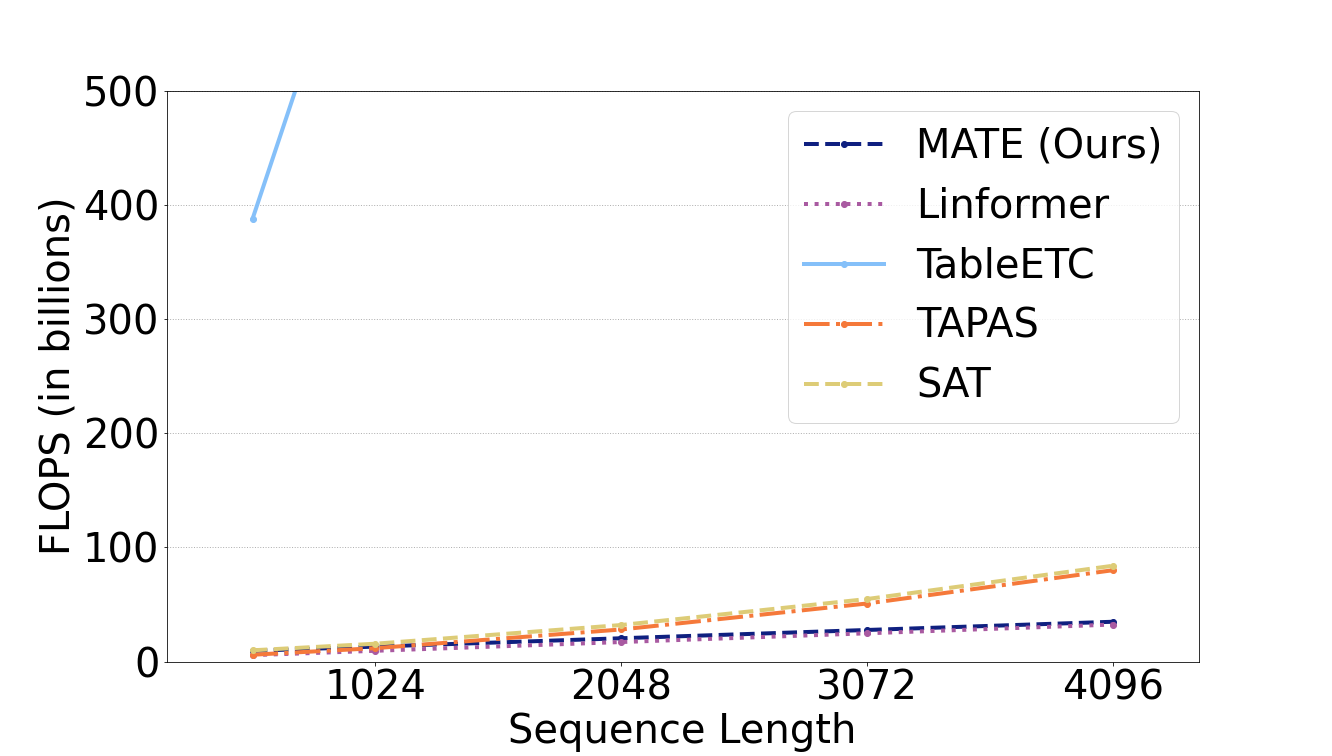}
    \caption{Comparison of inference FLOPS obtained from execution graph. While \tableetc is linear, relative attention adds a high computation cost so keep it out of range in the figure.}
    \label{fig:flops}
\end{figure}

\begin{figure}
    \centering
    \includegraphics[width=.95\linewidth]{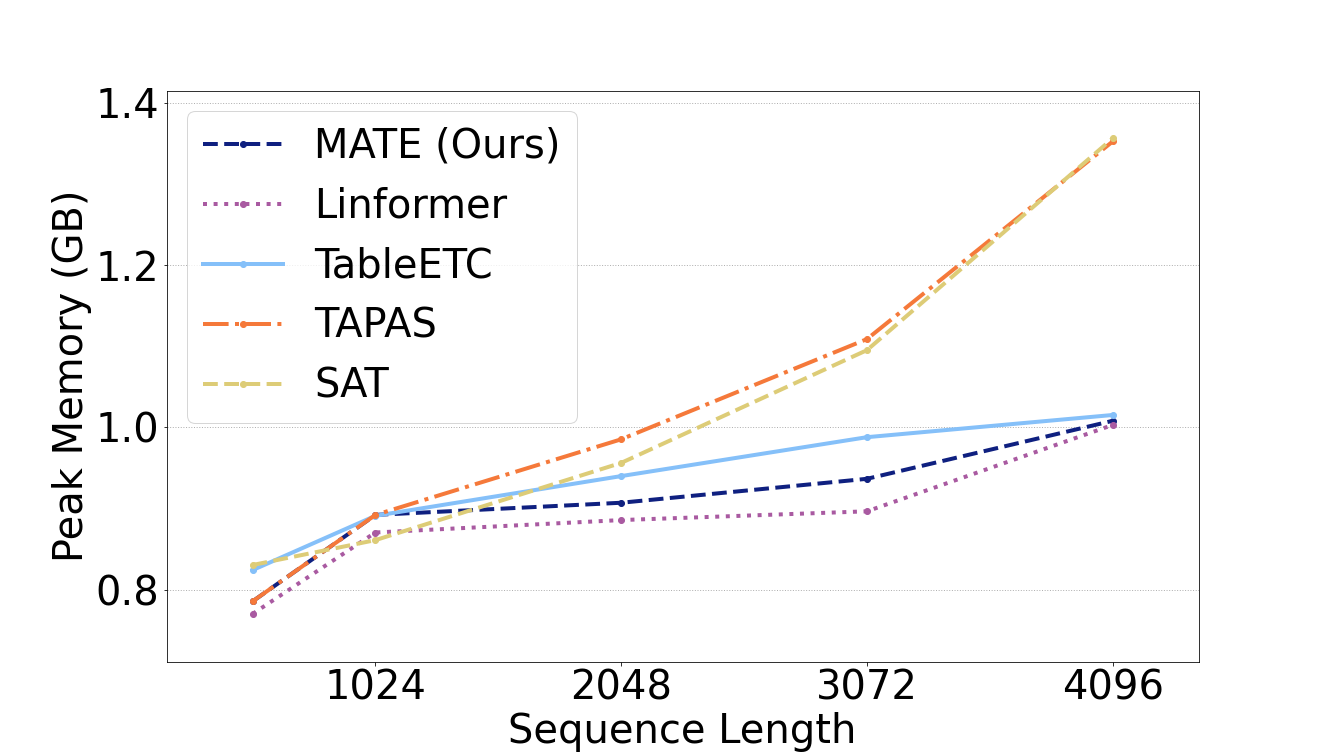}
    \caption{Comparison of the peak memory usage during CPU inference shows the linear asymptotic curve of the memory footprint of \model. }
    \label{fig:memory}
\end{figure}

\section{Development set results for \sqa, \wtq and \tabfact}
\label{sec:apx-dev-results}

We show in Table~\ref{tab:other_results_dev} the dev set results for all datasets we attempted, which show consistent results with the test set reported in the main paper.

\begin{table} %
\small
\centering
\resizebox{1\columnwidth}{!}{
\begin{tabular}{l|cccc}
\toprule
Model & \sqa \abr{all} & \sqa \abr{seq} & \wtq  & \tabfact \\
\midrule
\tapas & 64.9\err{0.5} & 40.0\err{1.0} & 41.6\err{1.0} & 76.9\err{0.4} \\
\model & \textbf{67.0}\err{0.1} & \textbf{43.2}\err{0.2} & \textbf{42.9}\err{0.6} & \textbf{77.5}\err{0.3} \\
\midrule
\tapas + CS & \textbf{68.0}\err{0.2} & \textbf{45.8}\err{0.3} & 46.2\err{0.2} & 81.0\err{0.1} \\
\model + CS & \textbf{68.0}\err{0.4} & 44.9\err{0.4} & \textbf{50.1}\err{0.7} & \textbf{81.3}\err{0.1} \\
\bottomrule
\end{tabular}
}
\caption{Dev results of using \model on other table parsing datasets. Errors are estimated with half the interquartile range over 5 runs.}
\label{tab:other_results_dev}
\end{table}

\section{Proof of Theorem \ref{theo:universal}}
\label{sec:apx-universal}
 
In this section we discuss the proof that \model are universal approximators of sequence functions.

\begin{theorem*}
If the number of heads is at least $3$ and the hidden size of the feed forward layer is at least $4$, then for any $f\in \mathcal{F}$ and 
$\epsilon \in \mathbb{R}_+$ 
there exists $\hat{f}\in\mathcal{T}_\model$ such that $||\hat{f} - {f}||_p < \epsilon$
\end{theorem*}

\begin{proof}
When the number of heads is at least $3$, there are at least $2$ heads of the same type. Fixing those two heads, we may restrict the value of the projection weights $W_V$ to be $0$ for the rest of the heads. This is equivalent to having only those two heads with the same attention pattern to begin with.
This restriction only makes the family of functions modelled by \model smaller. In a similar way, we can assume that the hidden size of the feed-forward layer is exactly $4$ and that the head size is $1$.

Note that the attention pattern of the two heads, regardless of its type contains a token (the first one) which attends to and from every other token. We also have that every token attends to itself. Then Assumption 1 of \citet{univapprox2020} is satisfied. 
Hence we rely on Theorem 1 of \citet{univapprox2020}, which asserts that sparse transformers with $2$ heads, hidden size $4$ and head size $1$ are universal approximators, which concludes the proof.
\end{proof}

\section{TensorFlow Implementation}
\label{sec:apx-code}

In figure~\ref{code} we provide an approximate implementation of \model in the TensorFlow library. For the sake of simplicity we omit how attention is masked between neighbor buckets for tokens in difference columns or rows. We also omit the tensor manipulation steps to reorder and reshape the sequence into equally sized buckets to compute attention across consecutive buckets. The full implementation will be part of the open source release.

\definecolor{codegreen}{rgb}{0,0.6,0}
\definecolor{codegray}{rgb}{0.5,0.5,0.5}
\definecolor{codepurple}{rgb}{0.58,0,0.82}
\definecolor{backcolour}{rgb}{0.95,0.95,0.92}

\lstdefinestyle{mystyle}{
    commentstyle=\color{codegreen},
    keywordstyle=\color{magenta},
    numberstyle=\tiny\color{codegray},
    stringstyle=\color{codepurple},
    basicstyle=\ttfamily\footnotesize,
    breakatwhitespace=false,         
    breaklines=true,                 
    captionpos=b,                    
    keepspaces=true,                 
    numbersep=5pt,                  
    showspaces=false,                
    showstringspaces=false,
    showtabs=false,                  
    tabsize=2
}

\lstset{style=mystyle}

\begin{figure*}%
\begin{center}
{\small
\begin{lstlisting}[language=Python]
import dataclasses
import tensorflow as tf

@dataclasses.dataclass
class MultiViewEmbedding():
  """Results of sorting and reshaping an embedding tensor.

  Different views of the tensor created to facilitate attention across tokens
  from global/long parts. First two dimensions are `batch_size` and `num_heads`:

  Attributes:
    full: <float32>[..., seq_length, embedding_size]. 
      Original tensor without any bucketing.
    global: <float32>[..., global_length, embedding_size].
    long: <float32>[..., long_length/radius, radius, embedding_size]
    window: <float32>[..., long_length/radius, 3*radius, embedding_size]
      Same as `long` but also a rotation to the left and right, in order
      to achieve attention to the previous and next bucket.
  """
  full: tf.Tensor
  global: tf.Tensor
  long: tf.Tensor
  window: tf.Tensor

def multi_view_attention(
    Q: MultiViewEmbedding,
    K: MultiViewEmbedding,
    V: MultiViewEmbedding,
    embedding_size: int,
    global_length: int,
    long_length: int,
    num_heads: int,
):
  # <float32>[batch_size, num_heads, global_length, sequence_length]
  attention_prob_from_global = tf.nn.softmax(tf.einsum(
      'BHFE,BHTE->BHFT', Q.global, K.full) / sqrt(embedding_size))
  # <float32>[batch_size, num_heads, long_length, global_length]
  attention_score_to_global = tf.einsum('BHNFE,BHTE->BHNFT',
      Q.long, K.global) / sqrt(embedding_size)
  # <float32>[batch_size, num_heads, long_length, 3 * radius]
  attention_score_to_window = tf.einsum('BHNFE,BHNTE->BHNFT',
      Q.long, K.window) / sqrt(embedding_size)
  # <float32>[batch_size, num_heads, long_length, global_length + 3 * radius]
  attention_prob_from_long = tf.nn.softmax(tf.concat(
      [attention_score_to_global, attention_score_to_window], axis=-1))
  attention_prob_to_global = attention_prob_from_long[..., :global_length]
  attention_prob_to_window = attention_prob_from_long[..., global_length:]

  # <float32>[batch_size, num_heads, global_length, embedding_size]
  context_layer_from_global = tf.einsum('BHFT,BHTE->BHFE',
      attention_prob_from_global, V.full)
  # <float32>[batch_size, num_heads, long_length / radius, radius, embedding_size]
  context_layer_to_global = tf.einsum('BHNFT,BHTE->BHNFE',
      attention_prob_to_global, V.global)
  # <float32>[batch_size, num_heads, long_length / radius, radius, embedding_size]
  context_layer_to_window = tf.einsum('BHNFT,BHNTE->BHNFE',
      attention_prob_to_window, V.window)
      
  context_layer_from_long = tf.reshape(
      context_layer_to_first + context_layer_to_window, 
      [-1, num_heads, long_length, embedding_size])
  return tf.concat(
      [context_layer_from_global, context_layer_from_long], axis=-1)

\end{lstlisting}
}
\caption{Implementation of \model in TensorFlow. The creation of \texttt{MultiViewEmbedding} is ommited and relies on \texttt{tf.gather} for ordering the input. We also omit the use of the input mask and column and row index to further mask the sparse attention matrix.}
\label{code}
\end{center}
\end{figure*}